\documentclass{article}
\usepackage{ijcai18}

\usepackage{times}
\usepackage{xcolor}
\usepackage{soul}
\usepackage[utf8]{inputenc}
\usepackage[small]{caption}

\usepackage{latexsym}
\usepackage{amssymb}
\usepackage{amsmath}
\usepackage{amsthm}
\usepackage[linesnumbered,ruled]{algorithm2e}
\usepackage{multirow}
\usepackage{tikz,pgffor}

\usepackage{pgfplots}
\usepackage{tikz}
\pgfplotsset{compat=newest}

\newcommand{\floor}[1]{{\left\lfloor #1 \right\rfloor}}
\newcommand{\ceil}[1]{{\left\lceil #1 \right\rceil}}

\newcommand{\Nset}{\mathbb{N}}

\newcommand{\Exp}{\mathbb{E}}

\newcommand{\game}{\mathcal{G}}

\newcommand{\val}{\mathit{Lev}}

\newcommand{\histories}{\mathcal{H}}

\newcommand{\pn}[1]{[#1]}

\newcommand{\Damage}{\mathit{Damage}}

\newcommand{\e}[1]{\langle #1 \rangle}

\newcommand{\dev}{\mathit{Dev}}

\newtheorem{theorem}{Theorem}

\newtheorem{definition}{Definition}
\newtheorem{example}{Example}

\renewenvironment{itemize}
{\begin{list}{$\bullet$}{
			\leftmargin=4mm
			\labelwidth=2.5mm
			\labelsep=2mm
			\itemsep=0mm
	}}{\end{list}}

\title{Synthesizing Efficient Solutions for Patrolling Problems\\ in the Internet Environment}

\author{
Tom\'{a}\v{s} Br\'{a}zdil, 
Anton\'{\i}n Ku\v{c}era, 
Vojt\v{e}ch \v{R}eh\'{a}k\thanks{Supported by grant No.~P202/12/G061, Czech Science Foundation. 
	Access to computing and storage facilities owned by parties and projects contributing to the National Grid Infrastructure MetaCentrum provided under the programme "Projects of Large Research, Development, and Innovations Infrastructures" (CESNET LM2015042), is greatly appreciated.}
\\ 
Faculty of Informatics, Masaryk University\\
Brno, Czech Republic\\
\{brazdil,kucera,rehak\}@fi.muni.cz
}

\begin{document}

\maketitle

\begin{abstract}	
We propose an algorithm for constructing efficient patrolling strategies in the Internet environment, where the protected targets are nodes connected to the network and the patrollers are software agents capable of detecting/preventing undesirable activities on the nodes. The algorithm is based on a novel compositional principle designed for a special class of strategies, and it can quickly construct (sub)optimal solutions even if the number of targets reaches hundreds of millions.
\end{abstract}

\section{Introduction}
\label{sec-intro}

A \emph{security game} is a non-cooperative game where the Defender (leader) commits to some strategy and the Attacker (follower) first observes this strategy and then selects a best response. 
In \emph{adversarial} security games, it is assumed that the Attacker not only knows the Defender's strategy, but can also observe the current positions and moves of the patrollers. This worst-case assumption is adequate also in situations when the actual Attacker's abilities are \emph{unknown} and robust defending strategies are required.

In this paper, we concentrate on adversarial patrolling in the \emph{Internet environment}, where the protected targets are fully connected by a network, and the patrollers are software agents freely moving among the targets trying to discover/prevent dangerous ongoing activities. 
We start by presenting two concrete scenarios\footnote{These scenarios should be seen just as \emph{examples} of possible application areas for our results, not as an exhaustive list.}  illustrating the considered class of security problems.
\smallskip

\noindent
\textbf{1.~Large-scale surveillance systems.}
Contemporary surveillance systems may comprise thousands (or even millions\footnote{According to IHS, there were 245 million professionally installed video surveillance cameras active and operational globally in 2014.}) of cameras watching complex scenes where real-time detection and alert are crucial. For example, crime detection systems assume fast response in case of crime or suspect detected. Also, they typically need sophisticated and computation intensive analytics (see, e.g.,~\cite{Cotton:Video-Analytics}) such as object detection, database retrieval of image data, etc. 
Simple object detection tools running on high-end GPUs are capable of processing hundreds full scale images per second \cite{RDGF:Object-Detection}. More detailed analysis, such as search of a detected face in a large database of suspects, causes even smaller number of processed images per second. This means that the delay caused by the analytics  may prevent \emph{simultaneous} real-time analysis of \emph{all} video streams in a large-scale surveillance system, and the system must intelligently switch among the streams in real-time.
Since the image processing time is substantially shorter than the intrusion time, there is a chance of achieving a good level of protection even if the system runs only a limited number of analytical processes concurrently. The crucial question is how to schedule the ``visits'' of these processes (patrollers) to the individual cameras (targets) so that the chance of discovering an ongoing intrusion is maximized. Since the analytics is not perfect, a patroller detects an ongoing intrusion in a currently visited target only with certain probability. Hence, the chance of successful intrusion detection increases if the target under attack is visited repeatedly before completing the attack.  
\smallskip

\noindent
\textbf{2.~Remote software protection.} 
Protecting software from man-at-the-end (MATE) tampering is a hard problem in general. In client-server systems, where the server part is considered trusted, \emph{continuous software updates} of the client software have been proposed as a promising technique for achieving the protection \cite{CT:Codebender-IEEESoftware,CMMN:distributed-tamper-detection}. Since completing a MATE attack requires a substantial amount of time and effort, the idea is to update some crucial components of the client software regularly (by the trusted server) so that the malicious ongoing analysis becomes useless and it must be restarted. If these updates are performed frequently enough, a MATE attack cannot be completed. However, the server's capacity is limited and the number of clients is typically very large (especially when the clients are split into independent submodules to decrease the update overhead). The question is how to design a suitable update policy for the server achieving a good protection against a MATE actively seeking for a weakly protected client. A detailed discussion of all relevant aspects can be found in \cite{BLM:Remote-SW-Protection} where a \emph{security game model} of the problem is designed. The patrollers are the update processes managed by the server, and the targets are the client software modules. Each target is assigned the time needed to complete a MATE attack and another natural number specifying its \emph{importance}. It is assumed that performing an update takes a constant time. Although the patrollers cannot detect an ongoing MATE attack at the currently updated target, it is assumed that a possible ongoing attack is always interrupted by the update. Furthermore, it is shown how to compute a \emph{positional} Defender's strategy for updating the targets, where the Defender's decisions depend only on the current positions (i.e., the tuple of currently visited modules) of the patrollers. In \cite{BLM:Remote-SW-Protection}, it is explicitly mentioned\footnote{A rigorous proof revealing insufficiency of positional strategies can be found in \cite{KL:patrol-regular}.} that positional strategies are \emph{weaker} than general \emph{history-dependent} strategies taking into account the whole history of previous updates. Hence, positional strategies are generally not optimal, and the computational framework of \cite{BLM:Remote-SW-Protection} does not allow for efficient construction of history-dependent strategies. This limitation is overcome in the presented work (see below). 

\smallskip

\textbf{Our contribution.} 
Before explaining our results, let us briefly summarize the key \emph{assumptions} about the considered class of problems which are reflected in the adopted game model (see Section~\ref{sec-defs}):
\smallskip

\noindent
\textbf{1.} \emph{The environment is fully connected}. This is the basic property of Internet underpinning our (novel) approach to strategy synthesis based on game decomposition. The use of (non)linear programming is completely avoided.
\smallskip
   
\noindent   
\textbf{2.} \emph{The patrollers are centrally coordinated.} The patrollers are software processes fully controlled by a server.
\smallskip
   
\noindent
\textbf{3.} \emph{The probability $p$ of recognizing an ongoing intrusion at a currently visited target is not necessarily equal to~$1$.} That is, the intrusion detection is not fully accurate in general, as in Scenario~1. Note that is Scenario~2, an intrusion (MATE attack) is recognized (interrupted) with probability~$1$.
\smallskip
   
\noindent   
\textbf{4.} \emph{The time needed to complete an intrusion may depend on a concrete target, and the individual targets may have different importance.} 
\smallskip

\noindent
\textbf{5.} \emph{The time needed to complete a patroller's activity is almost constant\footnote{Technically, this means that time is random but strongly concentrated around its expected value. Note that in Scenario~1, the real-time detection consists of two phases where a given image is first quickly classified as either harmless or potentially dangerous (in almost constant time), and dangerous images are subsequently enqueued for a more advanced analysis, such as face recognition in a database of suspects. This queue is processed separately (possibly using special hardware), and hence the second phase does not influence the assignment of patrolling processes.}.} This is satisfied in both scenarios.
\smallskip

\noindent
\textbf{6.} \emph{The number of targets is very large}. This influences mainly the methodology of algorithmic solutions. Algorithms for solving security games are mostly based on mathematical programming (see \emph{Related work}). This approach becomes \emph{infeasible} in our setting where the number of targets can easily reach a billion.
\smallskip

We adopt the adversarial setting, i.e., assume that the Attacker knows the Defender's strategy (the server program assigning the patrolling processes to targets) and can also determine the targets currently visited (e.g., by analyzing the network traffic or deploying malware into the server). 
Since there are no principal bounds on the duration of the patrolling task, our games are of infinite horizon. The adopted solution concept is \emph{Stackelberg equilibrium} (see, e.g, \cite{yin2010stackelberg}) where the Defender/Attacker corresponds to the leader/follower.

\smallskip

\noindent
Our \textbf{main results} can be summarized as follows:
\smallskip

\noindent
\textbf{A.} We give an \emph{upper bound} on the level of protection achievable for a given game structure and a given number of patrollers. Consequently, we can also derive a \emph{lower bound} on the number of patrollers needed to achieve a given level of protection. These bounds are valid for \emph{general} (i.e., history-dependent and randomized) strategies.  They are generally not tight, but good enough to serve as a ``yardstick'' for measuring the quality of the constructed Defender's strategies.
\smallskip

\noindent
\textbf{B.} We develop a novel \emph{compositional} approach to constructing Defender's strategies in Internet patrolling games. The method is based on splitting a given game into disjoint subgames, solving them recursively, and then combining the obtained solutions into a strategy for the original game. 
We evaluate the quality of the constructed strategies against the bounds described in~A., and show that our algorithm produces  \emph{provably} (sub)optimal solutions. A precise formulation is given in Section~\ref{sec-experiments}.
\smallskip

The running time of our algorithm is low. Instances with millions of targets are processed in  units of seconds (see Section~\ref{sec-experiments}). The only potentially costly part is solving a certain system of polynomial equations constructed by the algorithm, but this was always achieved in less than a second for all instances we analyzed (using Maple). The constructed strategies are randomized and history-dependent. We call them \emph{modular} because they use a bounded counter to count the units of elapsed time modulo certain constant. Hence, modular strategies are still easy to implement.

One may also ask whether some ``naive'' strategy synthesis method can produce strategies of comparable quality (i.e., whether our decomposition method is really worth the invested effort). Perhaps, the most straightforward way of constructing \emph{some} reasonable Defender's strategy is to compute a positional strategy where the probability of selecting a given vertex depends only on its importance and attack length, and it is chosen so that all vertices are protected equally well. We demonstrate that the strategies computed by our algorithm are \emph{substantially} better than these naively constructed ones.

\smallskip

\noindent
\textbf{Related work.} Most of the existing works about security games study either the problem of computing an optimal static allocation of available resources to the targets, or the problem of computing an optimal movement strategy for a mobile Defender. Security games with static allocation have been studied in, e.g., \cite{JKKOT:massive-security-games,JKKOT:optimal-resource-massive-security-games,PJMOPTWPK:Deployed-ARMOR,TRKOT:IRIS}. In patrolling games, the focus was primarily on finding locally optimal strategies for robotic patrolling units either on restricted graphs such as circles \cite{AKK:multi-robot-perimeter-adversarial,AgmonKK08-2}, or arbitrary graphs with weighted preference on the targets \cite{Basilico2009,Basilico2009-2}. Alternatively, the work focused on some novel aspects of the problem, such as variants with moving targets \cite{bosansky2011aamas,Fang2013}, multiple patrolling units \cite{Basilico2010}, or movement of the Attacker on the graph \cite{Basilico2009-2} and reaction to alarms \cite{MunozdeCote2013,BNG:patrolling-alarm}. Most of the existing literature assumes that the Defender is following a positional strategy that depends solely on the current position of the Defender in the graph and they seek for a solution using mathematical programming \cite{BGA:large-patrol-AI}. Few exceptions include duplicating each node of the graph to distinguish internal states of the Defender (e.g., in \cite{AKK:multi-robot-perimeter-adversarial} authors consider a direction of the patrolling robot as a specific state; in \cite{bosansky2012-aaaiss}, this concept is further generalized), or seeking for higher-order strategies in \cite{Basilico2009}. 
In \cite{ABRBKK:patrolling-eps-optimal}, an algorithm for computing an $\varepsilon$-optimal strategy for the Defender is designed, but this algorithm is of exponential complexity. The existing works on multi-agent patrolling mostly assume autonomous patrollers without a central supervision (see, e.g., \cite{ARSTMCC:Multi-agent-patrol}).

A patrolling game model for remote software protection has been proposed in \cite{BLM:Remote-SW-Protection}. The model is similar to ours, but the probability of detecting (interrupting) an ongoing intrusion is set to~$1$, which simplifies the strategy synthesis (repeated visits to the same target during an ongoing intrusion do not increase its protection). The constructed Defender's strategies are positional and computed by standard methods based on mathematical programming, which limits their scalability and does not lead to optimal solutions.

Continuous-time adversarial patrolling games with one patroller in fully connected environment have recently been studied in \cite{KST:Quasi-regular-sequences-SODA}. Based on the Defender's strategy, the Attacker selects which vertex to attack and for how long. The Attacker expected utility grows linearly with the time spent in the vertex but drops to 0 if being caught. The Defender's goal is to minimize the
Attacker's utility. This setup is technically different from ours, although it is also motivated by possible applications in Internet security problems (in \cite{KST:Quasi-regular-sequences-SODA}, no concrete scenarios illustrating the applicability of the considered model are given).

\section{The Game Model}
\label{sec-defs}

In this section we present our game-theoretic model of adversarial patrolling in the Internet environment  reflecting Assumptions~1.-6.{} formulated in Section~\ref{sec-intro}. 

\textbf{Preliminaries.} We use $\Nset_0$ and $\Nset$ to denote the sets of non-negative and positive integers, respectively. The set of all probability distributions over a finite set $M$ is denoted by $\Delta(M)$. The lower and upper integer approximations of a real number $a$ are written as $\lfloor a \rfloor$ and $\lceil a \rceil$, respectively. 
The number of elements of a set $A$ is denoted by $|A|$. A $k$-subset of $A$ is a subset of $A$ with precisely $k$~elements, and we use $A^{\e{k}}$ to denote the set of all $k$-subsets of~$A$. For $f : A \rightarrow B$ and $X \subseteq A$, we use $f|_{X}$ to denote the restriction of $f$ to $X$.

\textbf{Game structures.}
A \emph{game structure} is a tuple $\game = (V,d,\alpha,p)$, where $V$ is a finite set of \emph{vertices} (targets), \mbox{$d : V \rightarrow \Nset$} specifies the number of time units needed to complete an intrusion at a given vertex, \mbox{$\alpha : V \rightarrow \Nset$} is a \emph{cost} function specifying the importance of each vertex (a higher number means higher importance), and $p \in (0,1]$ is the probability of discovering an ongoing intrusion by a patroller visiting a vertex under attack. We assume that a patroller spends one unit of time when moving from vertex to vertex, which corresponds to performing the patroller's activity at the previously visited vertex.

\textbf{Defeder's strategy.} Assume $\game$ is protected by $k \in \Nset$ patrollers (where $k \leq |V|$) centrally coordinated by the Defender who has a complete knowledge about the history of previously visited vertices. Based on the history, the Defender selects a $k$-subset of vertices where the patrollers are sent in the next round, and this decision can be randomized.

Formally, a \emph{Defender's strategy} is a function $\eta$ assigning to every history  $U_1,\ldots,U_\ell$, where $\ell \geq 0$ and  $U_i$ is the \mbox{$k$-subset} of vertices visited in round~$i$, a probability distribution over~$V^{\e{k}}$. Note that a $k$-subset of vertices visited in the first round is determined by $\eta(\varepsilon)$, where $\varepsilon$ is the empty history.  A strategy  $\eta$ is \emph{positional} if $\eta(U_1,\ldots,U_\ell)$ depends only on~$U_\ell$.

Each strategy $\eta$ determines a unique probability space over all \emph{walks}, i.e., infinite sequences $U_1,U_2,\ldots$ where $U_i \in V^{\e{k}}$ for all $i \in \Nset$, in the standard way.

\textbf{Attacker's strategy.} Depending on the observed history of visited vertices, the Attacker may choose to attack some vertex or wait. Formally, an \emph{Attacker's strategy} is a function \mbox{$\pi$} assigning to every history an element of $V \cup \{\bot\}$ such that whenever $\pi(U_1,\ldots,U_\ell) \neq {\bot}$, then for all $j < \ell$ we have that $\pi(U_1,\ldots,U_j)={\bot}$, i.e., the Attacker may attack at most once along a walk.

\textbf{Level of protection.} The aim of the Attacker is to maximize the \emph{expected damage}, i.e., the expected cost of a successfully attacked target, and the Defender aims at the opposite. Let us fix an Attacker's strategy $\pi$, and let $w = U_1,U_2,\ldots$ be a walk. The damage achieved by $\pi$ in $w$, denoted by $\Damage^\pi(w)$, is defined as follows:
\begin{itemize}
   \item If the Attacker does not attack along $w$, i.e., $\pi(U_1,\ldots,U_\ell) = {\bot}$ for all $\ell \in \Nset_0$, then $\Damage^\pi(w) {=} 0$.
   \item Otherwise, there is $\ell \in \Nset_0$ such that $\pi(U_1,\ldots,U_\ell) = v$, where $v \in V$ is the attacked vertex. 
   For a given $i \in \{1,\ldots,d(v)\}$, we say that $v$ is \emph{visited in round~$i$} (since the moment of initiating the attack) if $v \in U_{\ell+i}$. For each such $i$, the probability of discovering the ongoing attack is~$p$. The attack is successful if it remains undiscovered after \emph{all} visits to $v$ in the next $d(v)$ rounds (the individual trials are considered independent). That is, the probability of performing the attack successfully is equal to $(1-p)^c$, where $c$ is the total number of all $i \in \{1,\ldots,d(v)\}$ such that $v$ is visited in round~$i$. We put 
   \[
   \Damage^\pi(w) \ = \  (1-p)^c \cdot \alpha(v) 
   \]
   to reflect the importance of the attacked vertex $v$ (if $p = 1$ and $c = 0$, we put
   $\Damage^\pi(w) =  \alpha(v)$). 
\end{itemize}
Let $\eta$ be a Defender's strategy for $k$~patrollers and $\pi$ an Attacker's strategy. The \emph{expected damage} caused by $\pi$ when the Defender commits to $\eta$, denoted by $\Exp^\eta[\Damage^\pi]$, is the expected value of $\Damage^\pi$ in the probability space over the walks determined by $\eta$ (see above). Note that $\Exp^\eta[\Damage^\pi] \leq \alpha_{\max}$, where $\alpha_{\max}$ is the maximal cost assigned to a vertex of $\game$.

In Section~\ref{sec-intro}, we used the term ``protection'' instead of ``damage'', and said that the Defender aims at maximizing the protection rather than minimizing the damage. This original terminology seems more intuitive, and it will be used also in the rest of this paper. Formally, the \emph{level of protection} achieved by $\eta$ against $\pi$ is defined as
\[
   \val(\eta,\pi) \ = \ \alpha_{\max} - \Exp^\eta[\Damage^\pi] \,.
\]
Note that maximizing the level of protection is \emph{equivalent} to minimizing the expected damage.

Furthermore, the level of protection achieved by $\eta$ (against \emph{any}~$\pi$) is defined by  
\(
\val(\eta) = \inf_{\pi} \val(\eta,\pi)\,.
\)
Finally, the maximal level of protection achievable with $k$~patrollers is defined as 
\mbox{\(
\val_k \ = \ \sup_{\eta} \, \val(\eta)
\)}, 
where $\eta$ ranges over all Defender's strategies for $k$~patrollers. The underlying $\game$ will always be clearly determined by the context.
A Defender's strategy $\eta$ for $k$~patrollers is \emph{$\delta$-optimal}, where $\delta \geq 0$, if $\val(\eta) \geq \val_k - \delta$. A \mbox{$0$-optimal} strategy is called \emph{optimal}.

\section{The Bounds}

In this section we give an upper bound on $\val_k$ and a lower bound on the number of patrollers needed to achieve a given level of protection. These bounds are not always tight, but good enough for evaluating the efficiency of Defender's strategies computed by the algorithm of Section~\ref{sec-alg}. 

The bounds are obtained by solving a non-trivial system of exponential equations constructed for a given game structure. In general, the solution can only be computed by numerical methods, which is fully sufficient for our purposes. Contemporary mathematical software such as Maple can perform the required computations very efficiently even if for high parameter values.

Let $\game = (V,d,\alpha,p)$ be a game structure, and $\alpha_{\max}$ the maximal cost assigned to a vertex of $\game$. Furthermore, let $\varrho$ be a fresh variable. For every $v \in V$, we construct the equation
\[
  \varrho = \alpha_{\max} - \alpha(v)\cdot (1-p)^\floor{Q_v} \cdot (1-p\cdot (Q_v - \floor{Q_v}))
\]
where $Q_v$ is a fresh variable (if $p=1$, the equation is simplified into  $\varrho = \alpha_{\max} - \alpha(v) \cdot (1- Q_v)$). Let $\mathcal{L}_{\game}$ be the resulting system of equations. Note that $\mathcal{L}_{\game}$ has $|V|$ equations and $|V| + 1$ variables.

\begin{theorem}
	\label{thm-upper-new}
	Let $\game = (V,d,\alpha,p)$ be a game structure.
	\begin{itemize}	
	\item[a)] The level of protection achievable with a given number of patrollers~$k$ is bounded from above by $\varrho$ obtained by solving the system $\mathcal{L}_{\game}$ extended with the equation $k = \sum_{v\in V,Q(v)>0} Q_v/d(v)$.
	\item[b)] Let $\tau \geq 0$ be a desired level of protection. Consider the system of equations obtained by extending $\mathcal{L}_{\game}$ with the equation $\tau = \varrho$. If this system has no solution such that $Q_v \leq d(v)$ for all $v \in V$, then the level of protection $\tau$ is not achievable for an arbitrarily large number of patrollers (if $p =1$, the condition is restricted to $Q_v \leq 1$).
	Otherwise, the number of patrollers needed to achieve the level of protection $\tau$ is bounded from below by $\ceil{\sum_{v\in V,Q(v)>0} Q_v/d(v)}$, where the value of each $Q_v$ is obtained by solving the system.
	\end{itemize}
\end{theorem}
\begin{proof}%
	The proof is based on a careful analysis of ``uniform coverage'' of all vertices with \mbox{$k$-patrollers} reflecting the importance of individual vertices. Intuitively, $Q_v$ stands for the expected number of visits to $v$ in subsequent $d(v)$ steps. 
	
	Let us fix a game structure $\game$ and a number of patrollers $k \in \Nset$. Let $\eta$ be an optimal strategy for $k$~patrollers (the existence of $\eta$ follows by standard arguments).
	Let $T=\prod_{v \in V} d(v)$. For all $1 \leq i \leq T$ and $v\in V$, let $P_i(v)$ be the probability of $v$ appearing in the $i$-th subset visited by a walk when the Defender commits to $\eta$. Note that for each $1 \leq i \leq T$, we have $\sum_{v\in V} P_i(v) = k$ as there are $k$ patrollers assigned in each round.
	Hence, the sum of all $P_i(v)$'s (for all $1 \leq i \leq T$ and $v \in V$) is equal to $T\cdot k$.
	
	For every $v \in V$ and every $0 \leq j < T$, where $j$ is a multiple of $d(v)$, let $\pi_{v,j}$ be the Attacker's strategy such that $\pi_{v,j}(w) = v$ for every history $w$ of length $j$. 
	A trivial calculation reveals  
	\begin{equation} 
	\val(\eta,\pi_{v,j}) \ = \ \alpha_{\max} - \alpha(v) \cdot \prod_{i=1}^{d(v)} \left(1 - p \cdot P_{j+i}(v) \right) 
	\label{eq-lev}
	\end{equation}
	Further, let $Q_{v,j} = \sum_{i=1}^{d(v)} P_{j+i}(v)$. By analyzing Equation~\eqref{eq-lev}, it follows the $\val(\eta,\pi_{v,j})$ is maximized if the values $P_{j+i}(v)$ satisfy 
	\begin{itemize}
	\item $P_{j+i}(v)=1$ for all $i=\{1,\dots, \floor{Q_{v,j}}\}$, 
	\item $P_{j+i}(v) = Q_{v,j} - \floor{Q_{v,j}}$ for $i=\floor{Q_{v,j}}+1$, and 
	\item $P_{j+i}(v) = 0$ for all $i > \floor{Q_{v,j}}+1$. 
	\end{itemize}
	The right-hand side of Equation~\eqref{eq-lev} can then be rewritten to \mbox{$\alpha_{\max} - \alpha(v) \cdot (1-p)^\floor{Q_{v,j}}\cdot(1-p\cdot(Q_{v,j}-\floor{Q_{v,j}}))$}. In the best case, when all $P_i(v)$ were distributed optimally by $\eta$, all $Q_{v,j}$ are the \emph{same} for a given $v \in V$, and the same level of protection $\varrho$ is achieved for every vertex~$v$. Thus, we use just $Q_{v}$ instead of $Q_{v,j}$ and 
	obtain the system $\mathcal{L}_\game$. 
	
	Further, the sum of all $P_i(v)$'s, for all $1 \leq i \leq T$ and $v \in V$, is equal to $T \cdot k$ (see above). By reordering the summands, the sum can be rewritten into $\sum_{v \in V} T/d(v)\cdot Q_v$. Thus, we obtain $k = \sum_{v \in V} Q_v/d(v)$.
	Observe that $0 \leq Q_v \leq d(v)$, and if $p = 1$, we have that $0 \leq Q_v \leq 1$, because then it does not make any sense to visit the same vertex more than once.
	
Let us note that when solving the systems considered in Theorem~\ref{thm-upper-new}, 
it may happen that some $Q_v$'s become negative. This happens if (and only if) the importance of some vertices is so low that even if the patrollers do not visit them at all, they are still protected better than the other vertices (this also explains why the sum $\sum_{v\in V,Q(v)>0} Q_v/d(v)$ is taken only over positive $Q(v)$'s). The system of Theorem~\ref{thm-upper-new}(b) may have no eligible solution in situations when $p$ is so small that the protection $\tau$ is not achievable even if each $v$ is visited with probability one in every step. 
If $p=1$, no eligible solution is induced only by $\tau > \alpha_{\max}$.
\end{proof}

\section{Modular strategies}
\label{sec-modular}
In this section, we introduce \emph{modular strategies} and the associated  \emph{compositional principle} which are the cornerstones of our strategy synthesis algorithm.

\begin{definition}
	Let $\game = (V,d,\alpha,p)$ be a game structure, and~$c_{\game}$ the least common multiple of all $d(v)$ where $v \in V$.
	A Defender's strategy $\eta$ for $\game$ is \emph{modular} if for every $h \in \histories$, the distribution $\eta(h)$ depends only on $\ell \mathrm{~mod~} c_{\game}$, where $\ell$ is the length of~$h$.
\end{definition}

\noindent
Hence, a modular strategy ``ignores'' the precise structure of a history and takes into account only  $\ell \mathrm{~mod~} c_{\game}$. At first glance, this looks like a severe restriction substantially limiting the efficiency of modular strategies. Surprisingly, this intuition turns out to be largely incorrect. As we shall see in Section~\ref{sec-experiments}, the level of protection achievable by modular strategies approaches and in some cases even \emph{matches} the upper bound of Theorem~\ref{thm-upper-new} which proves their \emph{(sub)optimality} among \emph{general} strategies.  Also note that modular strategies require only a bounded counter as auxiliary memory, which makes them easy to implement. 
To simplify our notation, we formally consider modular strategies as functions $\eta : \Nset_0 \rightarrow \Delta(V^{\e{k}})$, where $\eta(\ell)$ depends just on $\ell \mathrm{~mod~} c_{\game}$.

The main advantage of modular strategies is their \emph{compositionality}. A modular strategy for a given game structure $\game$ can be constructed by decomposing $\game$ into pairwise disjoint substructures $\game_1,\ldots,\game_m$, computing modular strategies for these substructures recursively, and then combining them into a modular strategy for~$\game$. The algorithm presented in Section~\ref{sec-alg} works in this way. Now we describe the composition step in greater detail.

Let $\game = (V,d,\alpha,p)$, and let $k$ be the number of patrollers protecting $\game$. Assume that we already managed to decompose $\game$ into smaller disjoint substructures $\game_1,\ldots,\game_m$, where $\game_i = (V_i,d|_{V_i}, \alpha|_{V_i},p)$, so that, for all $n \leq k$ and $i \leq m$, an efficient modular strategy $\eta_i\pn{n}$ for $\game_i$ and $n$~patrollers together with $\val(\eta_i\pn{n})$ can already be efficiently computed. In order to combine the constructed strategies into a strategy for $\game$, we need to solve the following problems: 
\begin{itemize}
	\item How to assign the $k$ available patrollers to $\game_1,\ldots,\game_m$ ?
	\item How to compose the modular strategies constructed for $\game_1,\ldots,\game_m$ into a modular strategy $\eta$ for $\game$?
\end{itemize}

\smallskip

\noindent
\textbf{Assigning patrollers to substructures.}
Clearly, the number of patrollers assigned to $\game_i$ should not exceed $|V_i|$. It is also clear that if $k$ is smaller than $m$, the assignment cannot be deterministic, because otherwise some substructures would remain completely unprotected, and the Attacker could attack them without any risk (recall the Attacker knows the Defender's strategy). So, we need to assign the $k$~patrollers to $\game_1,\ldots,\game_m$ \emph{randomly} in general. Formally, an \emph{assignment} for $k$~patrollers and $\game_1,\ldots,\game_m$ is a probability distribution $\beta$ over the set of all \emph{eligible} allocations $\vec{\kappa} \in \Nset_0^m$, where the components of $\vec{\kappa}$ sum up to~$k$ and $\vec{\kappa}_i \leq |V_i|$ for all $i\leq m$.

\smallskip 

\noindent
\textbf{Composing strategies constructed for substructures.} Assume we already constructed a suitable assignment $\beta$ for $k$ patrollers and $\game_1,\ldots,\game_m$. By our assumption, for all $n \leq k$ and $i \leq m$, an efficient modular strategy $\eta_i\pn{n}$ for $\game_i$ and $n$~patrollers can already by constructed. Our task is to construct a suitable modular strategy $\eta$ for $\game$ and $k$~patrollers. For every $\ell \in \Nset_0$, the outcome of $\eta(\ell)$ is determined as follows: 
\begin{itemize}
  \item First, some eligible allocation $\vec{\kappa} \in \Nset_0^m$ is selected randomly according to $\beta$ (independently of $\ell$). 
  \item Then, for each $i \in \{1,\dots,m\}$, we independently select a \mbox{$\vec{\kappa}_i$-subset} $U_i$ of $V_i$ according to $\eta_i\pn{\vec{\kappa}_i}(\ell)$. Thus, we obtain a \mbox{$k$-subset} \mbox{$U_1 \cup \cdots\cup U_m$} of $V$, which is the (random) outcome of $\eta(\ell)$.
\end{itemize}
Observe that the above definition makes a good sense because all $\eta_i\pn{\vec{\kappa}_i}$ are \emph{modular} strategies\footnote{If $\eta_i\pn{\vec{\kappa}_i}$ were general %
	 strategies, they could not be combined in such a simple way, because in each step, a different number of patrollers (including zero) can be assigned to a given $\game_i$, and hence the history produced by $\eta$ would not have to contain a valid history of $\eta_i\pn{\vec{\kappa}_i}$. So, it would not be clear how to simulate $\eta_i\pn{\vec{\kappa}_i}$.}. 
\smallskip

Hence, a concrete strategy synthesis algorithm based on the introduced decomposition principle (such as the one presented in Section~\ref{sec-alg}) must implement the following:
\begin{itemize}
	\item a \emph{decomposition procedure} which, for given $k$ and $\game$, either decomposes $\game$ into disjoint substructures to be solved recursively, or computes a suitable modular strategy for $\game$ and $k$ patrollers;
	\item an \emph{assignment procedure} which, for given $k$ and $\game_1,\ldots,\game_m$, computes an assignment for $k$ patrollers and $\game_1,\ldots,\game_m$.
\end{itemize}
These procedures may reflect different decomposition tactics apt for specific classes of instances.

\section{Examples} In Section~\ref{sec-alg}, we design a concrete decomposition algorithm based on the presented method and evaluate the quality of the constructed strategies. To get a preliminary intuition about its functionality, and to illustrate the notions introduced above, we present two simple examples.

\begin{example}\rm
	\label{exa-one}
	Let us consider $\game$ where $V = \{v_1,v_2,v_3\}$, \mbox{$d(v_i) = 2$} and $\alpha(v_i)=1$ for all $v_i$'s, and $p = 1$. We aim at protecting $\game$ with just one patroller. This appears like a trivial task---we have three equally important vertices with attack length~$2$, so a naive attempt is to choose the visited vertex uniformly at random in each step. The level of protection achieved in this way is equal to $1/3 + 2/3 \cdot 1/3 = 5/9$ (if a given vertex is attacked, it is visited in the first step with probability $1/3$, and if this does not happen (with probability $2/3$), there is another chance of $1/3$ in the second step). 
	
	If we try to ``improve'' the above strategy naively by selecting randomly only between the two vertices not visited in the previous step, the level of protection becomes \emph{worse}. If, say, $v_1$ is visited in step~$i$, then in step $i+1$ we choose just between $v_2$ and $v_3$, and in step $i+2$ we choose $v_1$ with probability $1/2$. So, the achieved level of protection is actually equal to $1/2$.
	
	So, can we do better than $5/9$? If so, what is the \emph{best} achievable level of protection? Given the simplicity of $\game$, the answers are perhaps surprising. 
	
	The decomposition algorithm of Section~\ref{sec-alg} starts by splitting $V$ into $V_1 = \{v_1,v_2\}$, $V_2 = \{v_3\}$, and constructing two trivial Defender's strategies $\eta_1$ and $\eta_2$ for one patroller in $V_1$ and $V_2$. The strategy $\eta_1$ selects $v_1$/$v_2$ with probability one at all even/odd steps, and $\eta_2$ always selects $v_3$. To combine $\eta_1$ and $\eta_2$, we need to fix an appropriate assignment $\beta$ for the two eligible vectors $(1,0),(0,1)$. That is, we need to determine $x \in [0,1]$ such that $\beta(1,0) = x$ and $\beta(0,1) = 1-x$. Now what is an appropriate value for $x$? The resulting strategy $\eta$ protects the vertices $v_1,v_2$ with probability $x$, but the vertex $v_3$ is protected with probability $1-x^2$ (in order to miss $v_3$, we need to miss it twice, each time with probability $1- (1-x) = x$).  Ideally, $\eta$ should protect all vertices equally well, which is satisfied when $x = 1-x^2$. This yields $x=(\sqrt{5}-1)/2 \approx 0.618 > 5/9$. Thus, the composed strategy $\eta$ achieves the level of protection equal to $(\sqrt{5}-1)/2$, and one can show that this is \emph{optimal}\footnote{The argument is non-trivial and specific to this instance.}. This also shows that the best achievable level of protection can be \emph{irrational}, even for very simple game structures.\qed    
\end{example}

\begin{example}\rm
	\label{exa-two}
	Let $\game$ be a game structure with six vertices $V = \{v_1,\ldots,v_6\}$ such that $d(v_i) = 2$ for all $v_i$'s, $\alpha(v_1) = \alpha(v_2) = 6$, $\alpha(v_3) = \alpha(v_4) = 3$, $\alpha(v_5) = \alpha(v_6) = 2$, and $p=1$. We aim at constructing a Defender's strategy for two patrollers. The algorithm of Section~\ref{sec-alg}  splits $V$ into the subsets $V_1 = \{v_1,v_2\}$,  $V_2 = \{v_3,v_4\}$, $V_3 = \{v_5,v_6\}$. For every $i \in \{1,2,3\}$, a trivial modular strategy $\eta_i$ for $\game_i$ and one patroller is constructed, where $\eta_i$ visits one vertex of $V_i$ at all even steps, and the other vertex of $V_i$ at all odd steps. 
	Now, the strategies $\eta_1,\eta_2,\eta_2$ are combined using an assignment $\beta$ which assigns at most one patroller to each $\game_i$, i.e., $\beta(1,1,0) = x$, $\beta(1,0,1) = y$, and $\beta(0,1,1) = z$, where $x+y+z =1$. Observe that the protection for the vertices of $V_1$, $V_2$, and $V_3$ achieved by $\eta$ is equal to
	\[
	6-(1-x-y)\cdot 6,\quad 6-(1-x-z) \cdot 3,\quad 6-(1-y-z) \cdot 2,
	\]
	respectively. Ideally, the above values should be equal. Thus, we obtain $x = 1/2$, $y = 1/3$, and $z = 1/6$. The level of protection achieved by the resulting strategy $\eta$ for $\game$ is equal to $5$, which matches the upper bound of Theorem~\ref{thm-upper-new}, i.e., $\eta$ is \emph{optimal}.\qed
\end{example}

\section{A Strategy Synthesis Algorithm}
\label{sec-alg}
In this section we design and evaluate a concrete strategy synthesis algorithm based on the compositional method presented in Section~\ref{sec-modular}. This algorithm is particularly apt for game structures where large subsets of vertices share the same attack length and the same importance weight. This applies to, e.g., surveillance systems and remote software protection systems\footnote{The number of IP camers or software modules can easily reach hundreds of millions, while the time needed to complete an intrusion or a software update typically ranges over a small interval of discretized time values. Similarly, the importance of a target usually ranges over a small set of discrete levels.} discussed in Section~\ref{sec-intro}. 
\smallskip

\textbf{Decomposition procedure.} For a given $\game = (V,d,\alpha,p)$, the decomposition starts by splitting the vertices of $V$ into pairwise disjoint subsets $U_{D,\tau}$ consisting of all $v \in V$ such that $d(v) = D$ and $\alpha(v) = \tau$ (for all $D$'s and $\tau$'s in the range of $d$ and $\alpha$, respectively). 
Then, each $U_{D,\tau}$ is further split into $\lceil |U_{D,\tau}|/D \rceil$ pairwise disjoint subsets of size precisely $D$ (these are called \emph{full}), and possibly one extra set with $(U_{D,\tau} ~\textrm{mod}~ D)$ elements (if $D$ does not divide $|U_{D,\tau}|$). These sets, constructed for all eligible $D$'s and $\tau$'s, are called the \emph{basic sets of~$\game$}, and they are not decomposed any further. Note that the decomposition of $\game$ into basic sets is independent of the number of patrollers assigned to protect $\game$.  

For every basic set $U = \{v_0,\ldots,v_{q-1}\}$ and every $k$, we need to compute a modular strategy $\mu$ for $U$ and $k$ patrollers. Recall that all vertices of $U$ have the same importance $\tau$ and the same attack length $D$, where $q \leq D$.  Let us first consider the simpler case when $q$ divides $D$. Imagine there are $k$ tokens, initially put on the vertices $v_0,\ldots,v_{k-1}$, which are then moved simultaneously around a circle formed by the vertices of $U$, where the successor of $v_i$ is  $v_{(i+1)~\textrm{mod}~ D}$. That is, after $\ell$ steps, the first token resides at $v_{\ell ~\textrm{mod}~ D}$, and the other tokens reside at the next $k-1$ vertices of the circle. The strategy $\mu(\ell)$ deterministically selects the $k$-tuple of vertices occupied by the tokens after $\ell$ steps.
Note that each token visits each vertex precisely $D/q$ times in $D$~consecutive steps.  

Now consider the general case when $q$ does not necessarily divide $D$.  In the first $\lfloor D/q \rfloor \cdot q$ steps, $\mu$ simulates $k$~tokens moving simultaneously around a circle similarly as above. In the remaining ($D ~\textrm{mod}~ q$) steps, a $k$-subset of $U$ is chosen uniformly at random (independently in each step). 
Formally, for every $\ell \in \{0,\ldots,D{-}1\}$, we have the following:
\begin{itemize}
	\item if $\ell < \lfloor D/q \rfloor \cdot q$, then $\mu(\ell)$ returns $\{v_{j_0},\ldots,v_{j_{k-1}}\}$ with probability one, where $j_n = (\ell+n) ~\textrm{mod}~ q$ for all $n \in \{0,\ldots, k{-}1\}$;
	\item otherwise, $\mu(\ell)$ returns a $k$-subset of $U$ chosen uniformly at random.
\end{itemize}
\smallskip

\noindent
\textbf{Assignment procedure.} This is the most advanced part of our algorithm. We need to construct an assignment $\beta$ for $k$ patrollers and the basic sets. We aim to construct $\beta$ so that the number of patrollers assigned to a given basic set $U$ is either $K_U$ or $K_U + 1$, where $K_U \in \Nset_0$ is a suitable constant depending only on $k$, the attack length, importance level, and the number of vertices of~$U$. The reason is that a wider variability in the number of patrollers assigned to $U$ would actually \emph{decrease}\footnote{This claim follows from the structure of modular strategies constructed for basic sets, and it can be proven rigorously.} the protection achieved for basic sets. 

Now we explain how to compute the constant $K_U$ and the probability $\lambda_U$ of sending precisely $K_U + 1$ patrollers to $U$.

Let us fix a basic set $U = \{v_1,\ldots,v_q\}$ where $\alpha(v_i) = \tau$ and $d(v_i) = D$ for all $v_i \in U$. If $\beta$ sends $K_U$ patrollers to $U$ with probability $(1-\lambda_U)$ and $K_U + 1$ patrollers with probability $\lambda_U$, then the expected number of patrollers sent to $U$, denoted by $E_U$, is equal to $K_U + \lambda_U$. Furthermore, the protection achieved for the vertices of $U$ is given by: %
{\scriptsize
\begin{equation}
\alpha_{\max} - \tau\cdot (1-p)^{K_U \cdot \lfloor D/q \rfloor} \cdot (1-p\cdot \lambda_U)^{\lfloor D/q \rfloor} 
      \cdot (1- p \cdot(K_U+\lambda_U)/q)^{(D ~\textrm{mod}~ q)}
\label{eq-protect-K-lambda}		
\end{equation}}%
Expression~\eqref{eq-protect-K-lambda} is obtained by a straightforward calculation omitted in here (possible subexpressions of the form $0^0$ are interpreted as $1$). 
Since we wish to protect all basic sets equally well, the above expression should produce the \emph{same} value for all $U \in W$. Hence, we can consider a system of equations stipulating that the above expressions are equal for all~$U$, and the sum of all positive $E_U$'s is equal to~$k$. The values for $K_U$ and $\lambda_U$ are obtained by solving this system.  Observe that \eqref{eq-protect-K-lambda} is parameterized by \emph{two} unknowns $K_U$ and $\lambda_U$, so we have more unknowns than equations. This is overcome by observing that $K_U = \lfloor E_U \rfloor$ and $\lambda_U = E_U - \lfloor E_U \rfloor$, which allows to use just \emph{one} variable $E_U$ instead of $K_U$ and $\lambda_U$. Furthermore, it may happen that a solution contains some negative $E_U$'s, which indicates that the importance of the vertices in $U$ is so low that $U$ is protected better than the other basic sets even if zero patrollers are assigned to~$U$. In this case, we simply remove $U$ from $\game$ and restart the algorithm.

\subsection{Evaluating the constructed strategies}
\label{sec-experiments}

For every Defender's strategy $\eta$ for $k$~patrollers, the \emph{relative deviation of $\eta$ from optimal strategy}, denoted by $\dev(\eta)$, is defined by \mbox{$(\val_k - \val(\eta))/\val(\eta)$}. Hence, $\eta$ is optimal iff \mbox{$\dev(\eta)=0$}. 

First, let us note that when all basic sets of $\game$ are full, then the achieved level of protection matches the upper bound of Theorem~\ref{thm-upper-new}, i.e., the strategy constructed by our algorithm is \emph{provably optimal}. For every game structure $\game$, there are two ``surrounding'' game structures obtained by decreasing/increasing the number of vertices in each $U_{D,\tau}$ to the nearest multiple of~$D$, for which our algorithm produces optimal strategies. With an increasing number of vertices, the levels of protection achieved for the surrouding game structures are closer and closer, which implies that $\dev(\eta)$ approaches zero. For a given $\delta >0$, one can easily compute a threshold such that, for every game structure where the number of vertices exceeds the threshold, the strategy $\eta$ constructed by our algorithm satisfies $\dev(\eta) < \delta$. Hence, our algorithm \emph{provably} produces strategies which are either optimal or very close to optimal, and this claim does \emph{not} require further experimental evidence. The aim of our experiments is to demonstrate that the time needed to solve the constructed system of equations is low (using Maple), and also to show that the protection achieved by our strategies is substantially better than the protection achieved by naively constructed strategies.

More concretely, we consider a surveillance system with three types of targets, where processing one image takes $0.1$~sec, an on-going intrusion is detected with probability $0.7$, the time needed to complete an intrusion is either $20$~secs, $2$~mins, or $15$~mins, and the target values are \$100000, \$130000, and \$400000, respectively. We analyze large instances where the number of targets of each type is $7000000 \cdot x$, $500000 \cdot x$, and $300000 \cdot x$, respectively, where the $x$ ranges from $1$ to $3$ with a step $0.01$. Hence, the total number of analyzed game structures is $200$ and the largest one has more than $23$ million vertices. The number of patrollers is set to $6000$. 
Since the time needed to compute $\eta$ was always negligible (the constructed systems of equations were solved by Maple in less than a second on an average PC), we do not report the running time details. 
The achieved levels of protection are plotted in Figure~\ref{fig-one}.
The levels of protection achieved by $\eta$ strategies differ from the theoretical upper bound by less than one dollar (for all instances). When we compare these strategies against naively constructed strategies $\sigma$, where in each round, the patrollers are identically distributed among the vertices so that all vertices are protected equally well, the difference between the levels of protection achieved by $\sigma$ strategies and the theoretical upper bound ranges between $\$157$ and $\$740$.

In Figure~\ref{fig-two}, we show the number of patrollers needed to achieve a given level of protection bounded by $270000$. Here we assume a fixed game with the same parameters as above where $x=1$, i.e., we have $7000000$, $500000$, and $300000$ vertices of the three types. The number of patrollers required by $\eta$ differs from the theoretical bound by at most one, while the naive strategies $\sigma$ need about $125\%$ 
of this amount on average. %
In the plot of Figure~\ref{fig-two}, the theoretical lower bound on the number of patrollers  and the number of patrollers required by $\eta$ are indistinguishable.

\begin{figure}
	\begin{tikzpicture}[font=\scriptsize]
	
	\begin{axis}[
	xmax=3.1,
	no markers,
	axis y line=left,axis x line=bottom,
	ytick={310000,308000,306000,304000},
	width=7.5cm,
	height=6cm,
	xlabel={$x$ such that the numbers of vertices are $7000000 \cdot x$, $500000 \cdot x$, and $300000 \cdot x$}, ylabel={level of protection by 6000 patrollers},
	scaled ticks=false, tick label style={/pgf/number format/fixed},
	]

	\addplot[line width=0.25pt,blue, dashed] file {pr-tb.dat};
	\addplot[line width=0.7pt,black,smooth, densely dotted] file {pr-v2.dat};
	\addplot[line width=0.25pt,red,smooth] file {pr-un.dat};

	\addlegendentry{bound of Theorem~\ref{thm-upper-new}}
	\addlegendentry{strategies $\eta$}
	\addlegendentry{strategies $\sigma$}
	
	\end{axis}
	
	\end{tikzpicture}
	\caption{The level of protection achieved by strategies.}
	\label{fig-one}	
\end{figure}
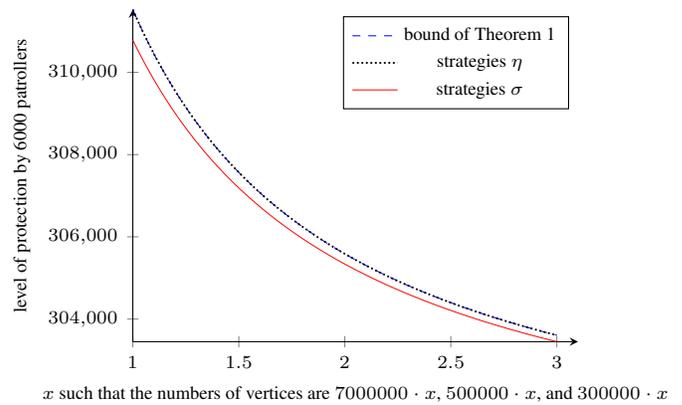

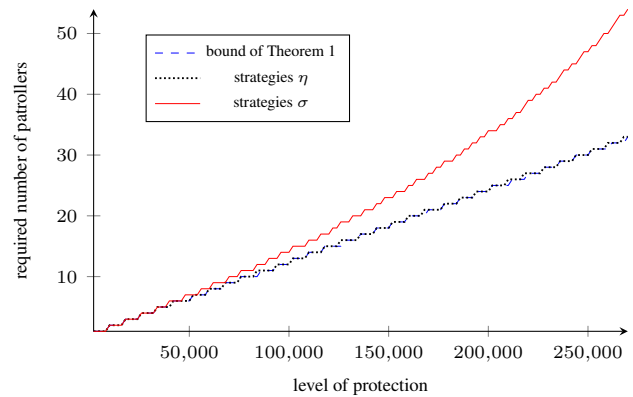
\begin{figure}	
	
\begin{tikzpicture}[font=\scriptsize]%

\begin{axis}[yscale=0.97, xscale=1.2,
scaled ticks=false, tick label style={/pgf/number format/fixed},
axis y line=left, axis x line=bottom,
ytick={0,10,20,30,40,50},
width=7.5cm,
height=6cm,
xlabel={level of protection}, ylabel={required number of patrollers},
legend style={font=\tiny,at={(0.4,.95)}},
]

\addplot[line width=0.25pt,blue,dashed] file {def-tb.dat};
\addplot[line width=0.7pt,black,densely dotted] file {def-v2.dat};
\addplot[line width=0.25pt,red] file {def-un.dat};

\addlegendentry{bound of Theorem~\ref{thm-upper-new}}
\addlegendentry{strategies $\eta$}
\addlegendentry{strategies $\sigma$}

\end{axis}
\end{tikzpicture}
	\caption{Achieving a given level of protection.}
\label{fig-two}
\end{figure}
	
\section{Conclusions}
\label{sec-concl}

The presented method can also be extended to more general models. For example, adapting our algorithm to a generalized model where $p$ is a \emph{function} from $V$ to $(0,1]$ is easy---the compositional principle stays the same, and in the algorithm, the equations used to compute an optimal assignment for the patrollers are slightly adjusted. Another direction for future research is to consider Attackers with limited capabilities.

The algorithm of Section~\ref{sec-alg} works well for game structures where the range of $d$ and $\alpha$ is relatively small compared to the number of vertices. For other classes of game structures, one may possibly develop other algorithms, but some decomposition principle similar to ours seems \emph{unavoidable} when the number of vertices reaches a certain threshold. Hence, we believe that the presented approach may trigger the development of the whole spectrum of decomposition techniques applicable to Internet security problems.


\end{document}